\documentclass[a4paper,11pt]{article}
\usepackage{amssymb,enumerate,amsmath,epsfig,amsthm,authblk,dsfont,bm}
\usepackage[T1]{fontenc}
\usepackage{smile}
\usepackage[margin=1in]{geometry}
\usepackage{setspace,color}
\usepackage{tikz}
\usepackage{epstopdf}
\usepackage{grffile}
\usepackage{makecell}
\usepackage{natbib}

\usepackage{hyperref}
\hypersetup{
	colorlinks=true,
	linkcolor=blue,
	citecolor=blue,
	filecolor=blue,      
	urlcolor=blue
}

\usepackage{lineno}

\providecommand{\keywords}[1]{\textbf{Key words:} #1}

\newcommand{\soft}{\mathrm{Softmax}}

\newcommand{\ReLU}{\mathrm{ReLU}}

\newcommand{\dt}{\Delta t}
\newcommand{\dx}{\Delta x}
\newcommand{\dy}{\Delta y}

\newcommand{\xxi}{\bm\xi}
\newcommand{\mylangle}{{\Big\langle}}
\newcommand{\myrangle}{{\Big\rangle}}

\newtheorem{theorem}{Theorem}[section]

\begin{document} 
\title{A Mathematical Explanation of Transformers}
\date{}
\author{
	Xue-Cheng Tai\thanks{NORCE Norwegian Research Centre, Nyg\r{a}rdstangen, NO-5838 Bergen, Norway Email: xtai@norceresearch.no, xuechengtai@gmail.com.}, 
	Hao Liu\thanks{Department of Mathematics, Hong Kong Baptist University, Kowloon Tong, Hong Kong. Email: haoliu@hkbu.edu.hk.},
	Lingfeng Li\thanks{Hetao Institue of Mathematics and Interdisciplinary Sciences (Shenzhen). Email: lilingfeng@himis-sz.cn.},
	Raymond H. Chan\thanks{School of Data Science and Department of Operations and Risk Management, Lingnan University, Tuen Mun, Hong Kong. Email: raymond.chan@ln.edu.hk.}
}
\maketitle

    \begin{abstract}
The Transformer architecture has revolutionized the field of sequence modeling and underpins the recent breakthroughs in large language models (LLMs). However, a comprehensive mathematical theory that explains its structure and operations remains elusive. In this work, we propose a novel continuous framework that rigorously interprets the Transformer as a discretization of a structured integro-differential equation. Within this formulation, the self-attention mechanism emerges naturally as a non-local integral operator, and layer normalization is characterized as a projection to a time-dependent constraint. This operator-theoretic and variational perspective offers a unified and interpretable foundation for understanding the architecture's core components, including attention, feedforward layers, and normalization. Our approach extends beyond previous theoretical analyses by embedding the entire Transformer operation in continuous domains for both token indices and feature dimensions. This leads to a principled and flexible framework that not only deepens on theoretical insight but also offers new directions for architecture design, analysis, and control-based interpretations. This new interpretation provides a step toward bridging the gap between deep learning architectures and continuous mathematical modeling, and contributes a foundational perspective to the ongoing development of interpretable and theoretically grounded neural network models.     
\end{abstract}


\section{Introduction}

Deep neural networks (DNNs) have revolutionized numerous fields, including natural language processing \cite{graves2013speech,young2018recent}, computer vision \cite{krizhevsky2012imagenet}, scientific computing \cite{lu2021learning,raissi2019physics,li2020fourier}, and medical diagnostics \cite{miotto2018deep,jiang2017artificial}. Among these architectures, Transformers \cite{vaswani2017attention} have recently emerged as particularly powerful tools, underpinning remarkable successes in large language models (LLMs) such as GPT-3 and GPT-4 \cite{vaswani2017attention}. Besides language processing \cite{zhu2025Transformers}, the Transformer and its variants have also been applied in applications, including image processing \cite{dosovitskiyimage,strudel2021segmenter,quan2024pseudo}, graph processing \cite{yun2019graph}, operator learning \cite{cheng2024mamba,bryutkin2024hamlet,liu2024prose,ye2024pdeformer}. 

Recently, a series of works have been conducted to study the theoretical foundation or interpretability of Transformers. Approximation and generalization error of Transformers were studied in \cite{furuya2024Transformers,takakura2023approximation,havrilla2024understanding}. It was shown in \cite{havrilla2024understanding,shen2025Transformers} that  Transformers are adaptive to data's low-dimensional structures. For interpretability, in \cite{lai2024attention}, it was shown that Transformers are cubic or high-order splines. It was shown in \cite{jelassi2022vision} that the vision Transformer can learn spatial structures of the dataset. The Transformer was interpreted as an ODE solver for multi-particle dynamical systems in \cite{dutta2021redesigning,geshkovski2023mathematical,lu2020understanding}. The importance of layer normalizaiton in Transformer was studied in \cite{kan2025stability}. A general introduction to Transformers with mathematical descriptions can be found in \cite{turner2023introduction}.

In this work, we will develop a mathematical  theoretical framework to explain the Transformer architecture  of \cite{vaswani2017attention}, which has achieved remarkable success in large language models (LLMs) and other sequence modeling tasks. We show that the Transformer can be interpreted as a discretization of the following continuous integro-differential equation:
\begin{align}
	\begin{cases}
		u_t=\underbrace{\mylangle \gamma(\xb,\cdot,t;u),V(\cdot,\yb,t;u)\myrangle_{\Omega_x} }_{\mbox{I: attention} } + \underbrace{\partial I_{S_{1}(\sigma_1(t),\sigma_2(t))}(u)}_{\mbox{II: layer normalization}} \\
		\hspace{3cm}+ \underbrace{\sum_{j=1}^{J}\left(\mylangle W_j(\cdot,\yb,t)u(\xb,\cdot,t)\myrangle_{\Omega_y}+ b_j(\xb,t)\right)
			+ \partial I_{S_2}(u)}_{\mbox{III: fully connected network}} \mbox{ for } t\in(0,T],\\
		u(\xb,\yb,0)=f(\xb,\yb),
	\end{cases}
	\label{eq.control-intr}
\end{align}
for $(\xb,\yb)\in\Omega_x\times\Omega_y$, where $\xb$ denotes the token index, $\yb$ the entry of token vectors, $\Omega_x$ and $\Omega_y$ are continuous domains, and $T>0$ is a fixed terminal time.
This formulation provides a rigorous mathematical interpretation of the Transformer and its key components, such as self-attention and feedforward layers. The definitions of attention, layer normalization, and the fully connected layer, along with detailed explanations of the domains $\Omega_x$, $\Omega_y$, and the control variables, are given in Section~\ref{sec.model}.
Equation~\eqref{eq.control-intr} highlights the non-local nature of attention and reveals the Transformer as a structured operator derived from a discretized variational principle.
We note that several previous works have also established connections between Transformer-based models and multi-particle dynamical systems \cite{dutta2021redesigning,geshkovski2023mathematical,lu2020understanding}. These studies interpret the original or modified Transformer models as discretizations of interacting particle systems in a time-continuous setting, treating interactions across spatial variables as particle interactions. In contrast, our formulation \eqref{eq.control-intr} is fundamentally different: we show that the Transformer arises as a discretization of a continuous integro-differential equation, offering a more unified operator-theoretic and variational perspective.

This perspective offers several advantages. First, it provides a unifying mathematical framework that connects diverse architectures, such as CNNs, UNets, and Transformers, under the common lens of differential and integral equations. This unification improves our understanding of the design principles underlying modern deep networks and provides insights for cross-architectural studies.

Second, by casting neural network architectures as time-stepping schemes of dynamical systems, our framework allows the systematic exploration of new architectures using well-established tools from numerical analysis. For instance, stability, convergence, and approximation properties of the underlying continuous models can inform the selection of network structures and hyperparameters, leading to more robust and interpretable models.

Third, this approach creates a principled pathway for embedding domain-specific knowledge, such as physical laws, geometric structures, or conservation principles, directly into the design of neural architectures, leading to architectures tailored for specific scientific or engineering tasks.

In summary, our formulation not only advances the theoretical understanding of deep neural networks but also provides actionable tools for the principled design of next-generation models. It bridges the gap between continuous mathematical modeling and discrete network implementation, establishing a foundation for informed, explainable, and application-aware neural network development.

This paper is organized as follows: We introduce the integral equation for Transformer with single-head attention in continuous setting in Section \ref{sec.model}. An operator-splitting based algorithm is proposed in Section \ref{sec.dis} to solve the integral equation. We show in Section \ref{sec.connections} that the discretized algorithm exactly recovers the Transformer encoder in \cite{vaswani2017attention}, and also discuss how to adapt the proposed algorithm to recover the architecture of Vision Transformer (ViT) \cite{dosovitskiyimage}. In Section \ref{sec.multihead}, we extend the integral equation and proposed algorithm to Transformer with multi-head attention. Extension to convolutional Transformer and its relation to Convolutional vision Transformer (CvT) \cite{wu2021cvt} are discussed in Section \ref{sec.convolution}. This paper is concluded in Section \ref{sec.conclusion}.

In the rest of this paper, we use bold lowercase letters to denote vectors, normal letters to denote scalars and functions, and calligraphic letters to denote operators.

\subsection*{Related works}
It is essential to examine the results of relevant literature and situate our approach within a historical context. Compared to Transformers, more studies can be found for deep neural networks (DNNs). A widely accepted interpretation views DNN as nonlinear high-dimensional function approximators \cite{yarotsky2017error,lu2021deep, zhou2020universality,chen2022nonparametric}. Under this view, the layers of a DNN perform successive nonlinear transformations that approximate complex mappings from input to output data. This approximation framework has been instrumental in explaining the expressive power of DNNs.
Another promising direction for explaining and designing neural network architectures involves interpreting them as discretizations of continuous-time dynamical systems governed by differential or integral equations \cite{chen2018neural,ma2020machine,benning2019deep,weinan2017proposal}. 
E et al. show that DNNs can be viewed as continuous dynamical systems in \cite{weinan2017proposal} and special discretization of continuous problems in \cite{ma2020machine}. In \cite{chen2018neural}, Chen et al. pointed out connections between residual networks and ordinary differential equations, and proposed Neural ODE. Benning et al. investigated the relations between DNNs and optimal control problems \cite{benning2019deep}. 
In this class of approach, neural network training is seen as an optimal control problem constrained by a continuous evolution equation. 
Related to the approach mentioned above, we especially want to mention some recent studies that have proposed a general framework for designing neural network architectures inspired by continuous-time dynamical systems \cite{tai2024pottsmgnet,liu2024double}. These methods leverage operator-splitting techniques to construct neural networks, which are so-called continuous neural networks, as discretizations of continuous differential equations. The continuous learning problem is an optimal control problem with this ``neural network'' as a constraint. In this framework, the evolution of hidden states in a network is governed by differential equations (continuous neural network). The goal of this continuous learning problem is to find optimal controls—represented by learnable parameters—within a discretized dynamical process. 

The use of operator-splitting allows these complex dynamics to be decomposed into simpler substeps, each of which corresponds to a specific type of operation (e.g., convolution, nonlinearity). These substeps are then unrolled into individual layers of a neural network. This interpretation not only clarifies the structure and purpose of different architectural components but also facilitates the incorporation of prior knowledge about the underlying data dynamics or physics.

Building on this theory, we have recently provided a rigorous mathematical explanation of the well-known UNet architecture \cite{tai2024mathematical}. We demonstrated that UNet can be viewed as a specific discretization of the following simple differential equation, where each encoder and decoder stage corresponds to a particular substep in the operator-splitting scheme: 
\begin{equation}
	\begin{cases}
		\frac {\partial u(\xb,t)}{\partial t}  =  W(\xb,t) \ast u(\xb,t) +  d(t) -\ln  \frac {u(\xb,t)} {1-u(\xb,t)}+ \partial I_{\Sigma}(u) , 
		\ (\xb,t) \in \Omega\times (0,T], \\ 
		u(\xb, 0)      = f(\xb), \ \xb \in \Omega,
	\end{cases}
	\label{eq.control0}
\end{equation}
where $\ast$ denotes convolution, $I_{\Sigma}$ is the indicator function of $\Sigma=\{u:u(\xb)\geq 0 \mbox{ for } \xb\in \Omega\}$, $f$ is a function defined in a domain $\Omega$ and $T$ is the final time chosen by the user for time variable $t$.  The learnable variables $\theta=\{W(\xb,t),  d(t)\}$ are regarded as control variables which will be learned in the training process. The mapping  $\mathcal{N}_\theta(f) \mapsto u(\xb,T)$ turns out to be the UNet in the continuous setting. 
This perspective not only elucidates the structure of the UNet architecture but also provides a foundation for systematic modifications and enhancements. A range of alternative interpretations and mathematically inspired designs of deep neural networks can be found in the literature, including \cite{ruthotto2020deep, benning2019deep, hagemann2025provable,ruiz2023neural, haber2018learning, he2019mgnet, lan2023dosnet, long2018pde, gregor2010learning, yang2018admm,altekruger2023neural}. For example, neural networks have been interpreted as optimal control problems in \cite{benning2019deep, ruiz2023neural}. Gregor and LeCun \cite{gregor2010learning} were among the first to connect deep learning with sparse coding through learned iterative shrinkage algorithms. Ruthotto and Haber \cite{ruthotto2020deep} proposed neural networks derived from partial differential equations (PDEs), while He and Xu \cite{he2019mgnet} introduced MgNet, which employs multigrid principles for feature extraction. Hagemann et al. \cite{hagemann2022stochastic} formulated stochastic normalizing flows as Markov chains to tackle inverse problems. Altekr{\"u}ger et al. \cite{altekruger2023conditional} analyzed the robustness of conditional generative model. Long et al. \cite{long2018pde} built PDENet based on finite difference schemes to recover PDE models from data. Yang et al. \cite{yang2018admm} constructed ADMM-CSNet by unrolling the ADMM optimization algorithm into a network architecture. Wang et al. \cite{wang2023diffusion} developed a diffusion-residual network inspired by convection-diffusion ODEs for classification tasks. Most recently, Martin et al. \cite{martin2024pnp} introduced a plug-and-play architecture using flow matching for image restoration. These pioneering efforts have provided essential insights and have directly inspired our recent work \cite{tai2024pottsmgnet, tai2024mathematical} and the present study.

\section{The Continuous Model}
\label{sec.model}

In the following, we explain the Transformer in the seminal work in \cite{vaswani2017attention} in the continuous setting and show later that proper discretizations of the proposed model will recover the Transformer exactly.

\subsection{A brief recap of the discrete Transformer encoder}
\label{sec:discrete-transformer}

The Transformer architecture introduced in~\cite{vaswani2017attention} is built around a core encoder block that processes sequential data through three principal components: a self-attention mechanism, layer normalization, and a feedforward network. In this section, we briefly recall its mathematical formulation, which will serve as the starting point for our continuous interpretation.

Let the input to a Transformer encoder layer be represented by a matrix
\[
\mathbf{u} = [\mathbf{u}_1^\top, \mathbf{u}_2^\top, \dots, \mathbf{u}_{N_x}^\top]^\top \in \mathbb{R}^{N_x \times N_y},
\]
where $N_x$ is the number of tokens and $N_y$ is the embedding dimension. Each row $\mathbf{u}_i \in \mathbb{R}^{N_y}$ corresponds to the embedded representation of a token.

\noindent\textbf{Self-Attention:}
The self-attention mechanism transforms $\mathbf{u}$ through learned weight matrices $\mathbf{W}^Q, \mathbf{W}^K, \mathbf{W}^V \in \mathbb{R}^{N_y \times N_y}$ to obtain query, key, and value matrices:
\[
\mathbf{Q} = \mathbf{u} \mathbf{W}^Q, \quad
\mathbf{K} = \mathbf{u} \mathbf{W}^K, \quad
\mathbf{V} = \mathbf{u} \mathbf{W}^V.
\]
The scaled dot-product attention is then computed as:
\[
\mathrm{Attention}(\mathbf{Q}, \mathbf{K}, \mathbf{V}) = \mathrm{Softmax}\!\left( \frac{\mathbf{Q} \mathbf{K}^\top}{\sqrt{N_y}} \right) \mathbf{V},
\]
where the softmax is applied row-wise. This operation allows each token to aggregate information from all tokens in the sequence, weighted by their compatibility.

\noindent\textbf{Layer Normalization:}
Following~\cite{ba2016layer}, layer normalization is applied per token (row-wise). Given an input vector $\mathbf{z} \in \mathbb{R}^{N_y}$, the normalized output is:
\[
\mathrm{LayerNorm}(\mathbf{z}) = \frac{\mathbf{z} - \mu}{\sigma} \odot \bgamma + \bbeta,
\]
where $\mu$ and $\sigma$ are the mean and standard deviation of $\mathbf{z}$, $\odot$ represents the elementwise product and $\bgamma, \bbeta \in \mathbb{R}^{N_y}$ are learnable parameters.

\noindent\textbf{Feedforward Network:}
The feedforward network consists of two linear transformations with a ReLU activation in between:
\[
\mathrm{FFN}(\mathbf{z}) = \mathrm{ReLU}\!\left(\mathbf{z} \mathbf{W}_1 + \mathbf{b}_1\right) \mathbf{W}_2 + \mathbf{b}_2,
\]
where $\mathbf{W}_1, \mathbf{W}_2 \in \mathbb{R}^{N_y \times N_y}$ and $\mathbf{b}_1, \mathbf{b}_2 \in \mathbb{R}^{N_y}$.

\noindent\textbf{Overall Encoder Block:}
A standard Transformer encoder layer combines these components with residual connections~\cite{vaswani2017attention}:
\[
\begin{aligned}
	&\mathbf{u}' = \mathbf{u} + \mathrm{Attention}(\mathbf{Q}, \mathbf{K}, \mathbf{V}), \\
	&\mathbf{u}'' = \mathrm{LayerNorm}(\mathbf{u}'), \\
	&\mathbf{u}''' = \mathbf{u}'' + \mathrm{FFN}(\mathbf{U}''), \\
	&\mathbf{u}_{\mathrm{out}} = \mathrm{LayerNorm}(\mathbf{u}''').
\end{aligned}
\]
This discrete, layer-wise processing will be our reference structure. In what follows, we show that this architecture can be viewed as a numerical discretization of a continuous dynamical system.

\subsection{The continuous setting for Transformers}
The interpretation of deep neural networks as discretizations of continuous-time dynamical systems has gained considerable traction in recent years~\cite{chen2018neural,ruthotto2020deep,tai2024pottsmgnet,weinan2017proposal}. These continuous viewpoints not only deepen theoretical understanding but also provide principled frameworks for architecture design and analysis.

In this section, we propose a control problem in continuous setting, named continuous Transformer, that reproduces the standard Transformer encoder through a structured discretization scheme. The continuous Transformer is a time-dependent integral equation. We will subsequently show that after discretizations by operator splitting for the time variable $t$ and spatial discretizations for the $\xb, \yb$ variables, the resulting splitting method recovers the Transformer architecture.   
The Transformer proposed in \cite{vaswani2017attention} consists of three components: attention, layer normalization, and feedforward networks. In the derivation of our control problem, we show that each component has a corresponding operation in the integral equation. We first focus on single-head attention. Our framework can be easily extended to the multi-head case, which will be detailed in Section \ref{sec.multihead}.

To make the explanation clearer and simpler, our notations and setup will follow \cite{vaswani2017attention}, but in a continuous counterpart. We consider bounded sets $\Omega_x=[0,L_x]^{d_x},\Omega_y=[0,L_y]^{d_y}$ for $L_x,L_y>0$ and positive integers $d_x,d_y$. When $d_x=d_y=1$, this corresponds to a sequence of bounded tokens and a bounded embedding in practice. Let $u(\xb,\yb,t)$ be a function of $\xb,\yb$ and time $t$. In applications of LLMs,  $\xb\in \Omega_x$ refers to the index of tokens, $\yb\in \Omega_y$ refers to entries of token vectors. Please refer to Appendix \ref{appendix:A} for explanations of the physical meanings of $\xb, \yb$, and $u$ for applications of Transformers to LLM.

As stated in \cite{vaswani2017attention}, attention is all you need for the Transformer. In fact, we show below that the attention layer has a very simple mathematical explanation in the continuous setting. It consists of integral transformations on the input functions to extract features with three different kernels. We use features extracted from the first two kernels to generate the attention score through a softmax operator, and then multiply the attention score by the feature extracted from the third kernel to compute the output. All three kernels in the integral transformation will be parameterized as learnable variables, see (\ref{eq:control0}). 

\subsection{Integral Transformations for Attention and Feature Extraction}
Let the functions $W^{Q}(\yb,\yt,t)$, $W^{K}(\yb, \yt,t)$, and $W^{V}(\yb,\yt,t)$ be three kernels defined on $\Omega_y\times\Omega_y\times [0,T]$, and they will be learned from data as shown in LLM applications. 
See Appendix \ref{appendix:A} for explanations of the physical meanings of $\xb, \yb, u(\xb,\yb, t)$.
For a given function $u(\xb,\yb, t) $, we define integral transformations:
\begin{align}
	&Q(\xb,\yb,t;u)=\mylangle W^{Q}(\cdot,\yb,t),u(\xb,\cdot,t)\myrangle_{\Omega_y}=\int_{\Omega_y}W^{Q}(\xi,\yb,t)u(\xb,\xi,t)d\xi ,
	\label{eq.atten.trans1}
	\\
	&K(\xb,\yb,t;u)=\mylangle W^{K}(\cdot,\yb,t),u(\xb,\cdot,t)\myrangle_{\Omega_y}=\int_{\Omega_y}W^{K}(\xi,\yb,t)u(\xb,\xi,t)d\xi , 
	\label{eq.atten.trans2}
	\\
	&V(\xb,\yb,t;u )=\mylangle W^{V}(\cdot,\yb,t),u(\xb,\cdot,t)\myrangle_{\Omega_y}=\int_{\Omega_y}W^{V}(\xi,\yb,t)u(\xb,\xi,t)d\xi.
	\label{eq.atten.trans3}
\end{align}
Here and later, we use $\langle \cdot, \cdot \rangle_{\Omega_y} $ to denote $L^2$ inner product on a given domain $\Omega_y $. Its induced norm will be denoted by $\|\cdot \|_{\Omega_y}$. 
Functions  $Q$ and $K$ will be used to generate attention scores, and $V$ is used to extract features from $u$. When all these functions are discretized over a pixel-type grid over $\Omega_x\times \Omega_y$, the discrete counterparts of these functions are represented by matrices and the discrete counterparts of the integral transformations are just standard matrix multiplications, see Section \ref{sec.space.dis}, especially formula (\ref{eq:dis-integral}). We compute the attention score by 
\begin{align}
	\gamma(\xb,\widetilde{\xb},t; u )=\text{Softmax}_{2}\left(\frac{1}{\sqrt{|\Omega_y|}}\mylangle Q(\xb,\cdot,t;u  ),\ K(\widetilde{\xb},\cdot,t;u ) \myrangle_{\Omega_y}\right) , 
	\label{eq.scale.soft}
\end{align}
where $\gamma(\xb,\widetilde{\xb},t )$ is defined on $\Omega_x\times \Omega_x\times [0,T]$, $\text{Softmax}_{2}$ represents the softmax function along the second dimension defined as:
\[ \text{Softmax}_{2}(a(\xb,\widetilde{\xb},t))=\frac{\exp(a(\xb,\widetilde{\xb},t))}{\int_{\Omega_x}\exp(a(\xb,\eta,t))d\eta} .  \label{eq.atten.trans4} \]
The output of the attention layer is then obtained by taking the inner product of the attention score $\gamma$ and the feature $V$: $\mylangle \gamma(\xb,\cdot,t;u),V(\cdot,\yb,t;u)\myrangle_{\Omega_x}.$
We emphasize that the first two integral transformations in (\ref{eq.atten.trans1})--(\ref{eq.atten.trans2}) are used to generate the attention score $\gamma(\xb, \tilde\xb, t;u)$ and the last integral transformation (\ref{eq.atten.trans3}) will be used for feature extractions. All three integral kernels will be learned during training, see explanations later in this section.

\subsection{Mathematical formulations for the Layer Normalization} 
The layer normalization \cite{ba2016layer} in the Transformer \cite{vaswani2017attention} and also many other neural networks \cite{ziaee2022batch,zhang2019lightweight,xu2019understanding,cui2022layer} can be mathematically described as a projection of a function to a set with given mean value $\sigma_1$ and variance $\sigma_2^2$.   
Define sets
\begin{align}
	&S_{1}(\sigma_1,\sigma_2)=\bigg\{u: \frac{1}{|\Omega_y|}\int_{\Omega_y}u(\xb,\xi,t)d\xi=\sigma_1,\nonumber\\
	&\hspace{5cm}\frac{1}{|\Omega_y|}\int_{\Omega_y}(u(\xb,\xi,t)-\sigma_1)^2d\xi=\sigma_2^2\bigg\},\\
	&S_2=\{u: u\geq 0\},
\end{align}
and their corresponding indicator functions 
\begin{align}
	I_{S_1(\sigma_1,\sigma_2)}(u)=\begin{cases}
		0 & \mbox{ if } u\in S_1,\\
		+\infty & \mbox{ otherwise,}
	\end{cases} \quad
	I_{S_2}(u)=\begin{cases}
		0 & \mbox{ if } u\in S_2,\\
		+\infty & \mbox{ otherwise.}
	\end{cases}
\end{align}
Later, we shall show that layer normalization is just a projection of a function into the set  $S_1(\sigma_1,\sigma_2)$, see Section \ref{sec.normalization.dis} and the ReLU activation function is just a projection of a function to the set $S_2$, see Section \ref{sec.fnn.dis}. 

\subsection{Our continuous Transformer}
In addition to the control variables $W^{Q}(\yb,\yt,t)$, $W^{K}(\yb, \yt,t)$, $W^{V}(\yb,\yt,t)$, let us also introduce control variables $\{W_j(\yt,\yb,t)\}_{j=1}^J$ on $\Omega_y\times \Omega_y\times [0,T]$, and $b(\xb,t),b_j(\xb,t)$ on $\Omega_x\times [0,T]$.  
We call the following continuous integral equation a {\bf Continuous Transformer}:
\begin{align}
	\begin{cases}
		u_t=\underbrace{\mylangle \gamma(\xb,\cdot,t;u),\ V(\cdot,\yb,t;u)\myrangle_{\Omega_x} }_{\mbox{I: attention} } + \underbrace{\partial I_{S_{1}(\sigma_1(t),\sigma_2(t))}(u)}_{\mbox{II: layer normalization}} \\
		\hspace{1cm}+ \underbrace{\sum_{j=1}^{J}\left(\mylangle W_j(\cdot,\yb,t), \ u(\xb,\cdot,t)\myrangle_{\Omega_y}+ b_j(\xb,t)\right)
			+ \partial I_{S_2}(u)}_{\mbox{III: fully connected network}}, \quad t\in (0,T], \\
		u(\xb,\yb,0)=f(\xb,\yb).
	\end{cases}
	\label{eq.control}
\end{align}
For simplicity and to align with common practice in Transformer implementations, we equip (\ref{eq.control}) with zero boundary conditions or periodic boundary conditions in both $\xb$ and $\yb$. This is commonly used in LLMs and image processing applications.

We shall show in the rest of this paper that the Transformer proposed in \cite{vaswani2017attention} is just a discretization of this continuous Transformer. 
In (\ref{eq.control}),  $f$ is some initial state, $T$ is a fixed time chosen by the user. Following conventions in the literature, we denote all the control variables as 
\begin{equation}
	\theta=\left\{W^Q(\xb,\yb,t),W^K(\xb,\yb,t),W^V(\xb,\yb,t),\{W_j(\xb,\yb,t),b_j(\xb, t)\}_{j=1}^J,\sigma_1(t),\sigma_2(t)\right\}. 
	\label{eq:theta}    
\end{equation}
The Continuous Transformer is the mapping: 
\begin{equation}
	\mathcal{N}_\theta: f \mapsto  u(\xb,\yb, T). 
	\label{eq:ntheta}
\end{equation}

In Section \ref{sec.dis}, we will discretize (\ref{eq.control}) by the operator-splitting method for the time variable and use a uniform grid for the spatial variables $\xb,\yb$.  After discretization, the continuous Transformer becomes the Transformer proposed in \cite{vaswani2017attention}. In relation to this equivalence, the operations in the right-hand side of (\ref{eq.control}) can be classified into three sets:
\begin{itemize}
	\item Operations in Set I correspond to attentions in Transformers. It computes the attention score $\gamma(\xb, \tilde \xb, t; u)$ and applies it to the extracted features in $V(\xb,\yb,t)$.
	
	\item Operations in Set II correspond to layer normalization. For each token location $\xb$, it requires function $u$ to have mean $\sigma_1$ and variance $\sigma_2^2$ along the $\yb$ direction. In other words, for each token, the corresponding series is expected to have mean $\sigma_1$ and variance $\sigma_2^2$.
	
	\item Operations in Set III correspond to feedforward networks with $J$ layers. The first component corresponds to linear layers, each of which contains one linear transformation $W_j$ and a bias $b_j$. The second component corresponds to activation functions. This term projects $u$ to $S_1$, which can be realized by ReLU functions.
\end{itemize}

\subsection{The continuous control problem for learning}
Our objective is to optimize $\theta$ so that given an input $f(\xb,\yb)$, the control problem will drive $u(\xb,\yb,T)$ to a desired state $v$. Specifically, given a dataset $\{(u_i,v_i)\}_{i=1}^B$, where $u_i$ is the input and $v_i$ is the target state, let $\ell(\cdot,\cdot)$ be a loss function measuring the discrepancy between its arguments.  
Let $\theta, \mathcal{N}_\theta$ be as defined in (\ref{eq:theta})--(\ref{eq:ntheta}), the training process with the loss function $\ell$ is in fact a process to optimize $\theta$ for the following minimization problem: 
\begin{align}
	\min_{\theta} \frac{1}{B}\sum_{i=1}^B\ell(\cN_{\theta}(u_i),v_i).
	\label{eq:control0}
\end{align}
This is essentially an optimization problem constrained by an integral differential equation of the following form: 
\begin{align}
	\min_{\theta} \frac{1}{B}\sum_{i=1}^B\ell(w_i, v_i), \mbox{ under constraint } w_i=\cN_{\theta}(u_i) \mbox{ solves the control problem (\ref{eq.control})}.
\end{align}

\section{Discretizations of the Continuous Problem}
\label{sec.dis}
Following the idea of PottsMGNet \cite{tai2024pottsmgnet}, in this section, we use some proper operator-splitting methods to discretize the time variable for (\ref{eq.control}). 
For a review about operator splitting methods, please refer to 	\cite{glowinski1989augmented,glowinski2017splitting,glowinski2016some}. 
Specifically, we discretize the temporal domain of (\ref{eq.control}) by sequential splitting, c.f. \cite[Sec. 2.2]{glowinski2016some}. The number of sequential splittings corresponds to the number of layers in Transformers. Spatial discretization for $\xb$ and $\yb$ will be discussed in Section \ref{sec.space.dis}. The discretization of $\xb$ determines the number of tokens in the input, and the discretization of $\yb$ determines the size of the embedding vector for each token.

\subsection{The overall time discretization}
\label{sec.learning.timedis}
Let $\{t^n\}_{n=0}^{N_t}$ be the set of time grids with time step $\dt=T/N_t$. We denote the numerical solution of (\ref{eq:control0}) at time $t^n$ by $u^n$. We will design a numerical scheme to approximate $\cN_{\theta}$. Specifically, for each $t^n$, we design a discrete propagator $\bar{\cN}^n$ propagating $u^{n-1}$ to $u^n$. Then, problem (\ref{eq.control}) is numerically approximated  by
$$
u(t^{N_t})\approx u^{N_t}=\bar{\cN}^{N_t}\circ \bar{\cN}^{N_t-1} \circ \cdots \circ \bar{\cN}^{1}(f),
$$
see Figure \ref{fig.architecture}(a) for an illustration. 

\begin{figure}[t!]
	\centering
	(a)\\
	\includegraphics[scale=0.22]{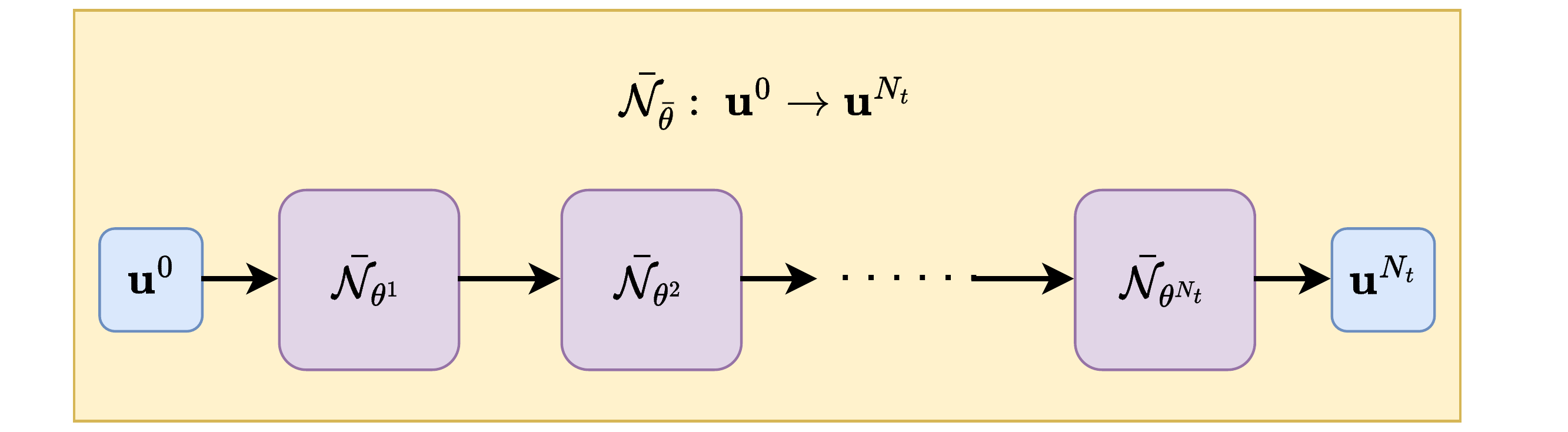}\\	
	(b)\\
	\includegraphics[scale=0.22]{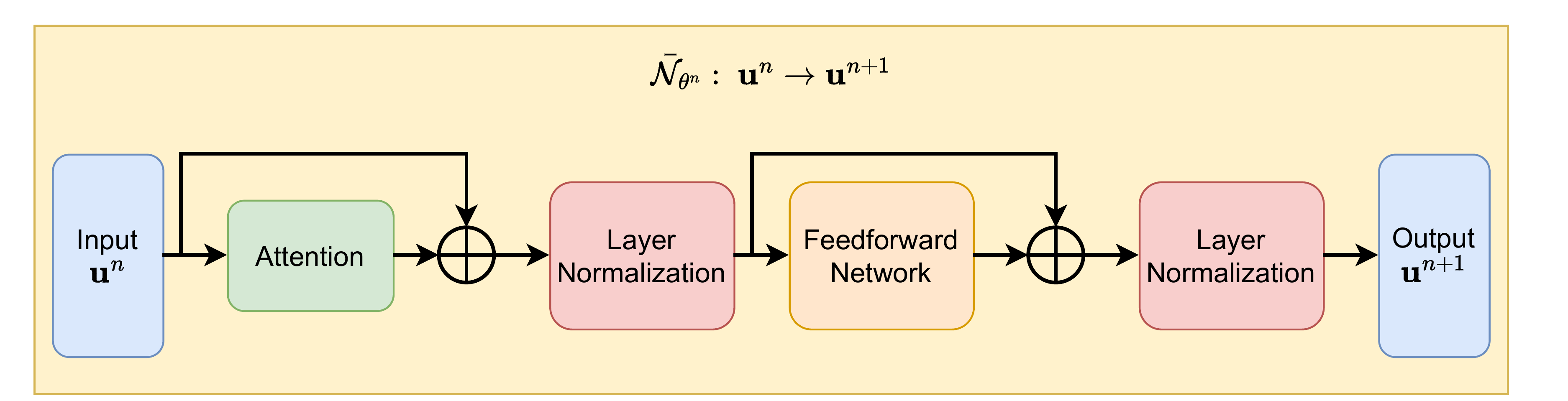}
	\caption{(a) Illustration of the whole evolution. (b) Illustration of the discretized operator-splitting scheme (\ref{eq.split.attention})--(\ref{eq.split.normal2}). }
	\label{fig.architecture}
\end{figure}

In each $\bar{\cN}^n$, time-discretized control variables $\theta$ will be used. We denote the control variables used in $\bar{\cN}^n$ by $\theta^n$. 
To show the dependence of $\bar{\cN}^n$ on the control variable $\theta^n$, we will also use $\bar{\cN}^n_{\theta^n}$ to denote the propagator $\bar{\cN}^n$.  
From the following subsections, c.f. \S\ref{sec.space.dis}, we will see that 
$$
\theta^n=\{W^{Q,n-1},W^{K,n-1},W^{V,n-1},\{W_j^{n-1},b_j^{n-1}\}_{j=1}^J,\sigma_1^{n-1},\sigma_2^{n-1}\}.
$$ 
Denote all of the discretized variables by $\bar{\theta}$, i.e. $\bar{\theta}= \{\theta^n\}_{n=1}^{N_t}$.  The whole evolution procedure from $f$ to $u^{N_t}$ by $\bar{\cN}_{\bar \theta}$,  i,e, 
$\bar{\cN}_{\bar \theta} =   \bar{\cN}^{N_t}_{\theta^{N_t}}\circ \bar{\cN}^{N_t-
	1}_{\theta^{N_t-1}} \circ \cdots \circ \bar{\cN}^{1}_{\theta^1}$.
Given a dataset $\{(u_i,v_i)\}_{i=1}^B$, then the discrete training problem is to solve the following  constrained optimization problem: 
\begin{align}
	\min_{\bar{\theta}} \frac{1}{B}\sum_{i=1}^B\ell(\bar{\cN}_{\bar{\theta}}(u_i),v_i).
	\label{eq.loss.dis}
\end{align}

\subsection{Time Discretizations by Operator Splitting Method}
\label{sec.timedis}
In the following sections, we illustrate our numerical scheme for propagator $\bar{\cN}^1_{\theta^1}$. Propagator $\bar{\cN}^n_{\theta^n}$ for $n=2,...,N_t$ is exactly the same by replacing $\theta^1$ by $\theta^n$.
Our numerical scheme is designed based on operator splitting methods. 
For simplicity, we take $\Delta t=1$. It is often used in data analysis and image processing literature and is equivalent to changing the network parameters properly.  
In the following, with a slight abuse of notation, we denote the time-discretized version of any time-dependent function \( a(\cdot, t) \) by \( a^n(\cdot) := a(\cdot, t^n) \), where \( t^n \) is the \( n \)-th time step.
Using the Lie scheme  \cite[Sec. 2.2]{glowinski2016some} as explained in  Appendix \ref{appendix:c}, given $u^0 =f$, we compute $u^{1}$ by $M=4+J$ substeps:

\noindent {\bf Substep 1}:
Solve $u^{1/M}$ from
\begin{align}
	&u^{1/M}-u^{0}= \mylangle \gamma^0(\xb,\cdot;u^{0}), \ V^0(\cdot,\yb;u^{0})\myrangle_{\Omega_x}.
	\label{eq.split.attention}
\end{align}

\noindent {\bf Substep 2}:  Compute $u^{2/M}$ from
\begin{align}
	u^{2/M}-u^{1/M}=\partial I_{S_1(\sigma_1^0,\sigma_2^0)}(u^{2/M}).
	\label{eq.split.normal1}
\end{align}

\noindent {\bf Substep 2+j }:  
Compute $u^{(2+j)/M}$ sequentially for $j=1,...,J-1$ from
\begin{align}
	u^{(2+j)/M}-u^{(1+j)/M}= \mylangle W_j^0(\cdot,\yb), \ u^{(1+j)/M}(\xb,\cdot)\myrangle_{\Omega_y}+ b^0_j(\xb) + \partial I_{S_2}(u^{(2+j)/M}).
	\label{eq.split.fnn}
\end{align}

\noindent {\bf Substep 2+J }:  
Compute $u^{(2+J)/M}$ from
\begin{align}
	u^{(2+J)/M}-u^{(1+J)/M}= \mylangle W_j^0(\cdot,\yb), \ u^{(1+J)/M}(\xb,\cdot)\myrangle_{\Omega_y}+ b^0_j(\xb).
	\label{eq.split.fnnfinal}
\end{align}

\noindent {\bf Substep 3+J}: 
Compute $u^{(3+J)/M}$ from
\begin{align}
	u^{(3+J)/M}= \frac{1}{2}(u^{(2+J)/M}+u^{2/M}).
	\label{eq.split.skip2}
\end{align}

\noindent {\bf Substep 4+J}: 
Compute $u^1$ from 
\begin{align}
	u^{1}-u^{(3+J)/M}=\partial I_{S_1(\sigma_1^0,\sigma_2^0)}(u^{1}).
	\label{eq.split.normal2}
\end{align}
Note that in (\ref{eq.split.attention})--(\ref{eq.split.fnnfinal}) and (\ref{eq.split.normal2}), the left hand side $u^{(k+1)/M}-u^{k/M}=\frac{u^{(k+1)/M}-u^{k/M}}{\Delta t}$ with $\Delta t=1$  is the finite difference approximation of $u_t$. 

Substep 1 to Substep $(4+J)$ evolve $u^0$ to $u^1$ by sequentially passing it through all operators in (\ref{eq.control}), corresponding to all components in Transformer from input to output. Substep 1 corresponds to the attention layer and the skip connection. Substep 2 corresponds to layer normalization. Substep 3 to Substep $(2+J)$ correspond to the $J$ layers in the fully connected network. Substep $3+J$ corresponds to the skip connection after the fully connected network. The last substep corresponds to the final layer normalization.  From here, we can already see that each substep in the Lie splitting scheme corresponds to one layer in the Transformer. In \cite{vaswani2017attention}, it used $J=2$ and thus $M=6$. The different layers in the Transformer is just a sequential computation of the substep solutions: 
\begin{multline}
	u^0 \quad \text{(input token)} 
	\mapsto u^{1/6} \text{(attention)} 
	\mapsto u^{2/6} \text{(normalization layer)} \\
	\mapsto u^{3/6} \mbox{ and } u^{4/6} \text{(fully connected layer with ReLU)} \\
	\mapsto u^{5/6} \text{(skip connection layer)} 
	\mapsto u^1  \text{(normalization layer)}.
\end{multline}

\begin{remark}
	In scheme (\ref{eq.split.attention})--(\ref{eq.split.normal2}), substep (\ref{eq.split.skip2}) is a relaxation step and substep (\ref{eq.split.normal1}) and (\ref{eq.split.normal2}) repeat $\partial I_{S_1(\sigma_1^0,\sigma_2^0)}$ twice. These two are numerical techniques commonly used to improve algorithm stability and accuracy. For example, it is shown in \cite{ketcheson2019relaxation} that applying relaxation technique (averaging step) to Runge-Kutta methods helps guarantee conservation or stability with respect to any inner-product norm. For another example, compared to Lie's splitting, the Strang splitting intentionally repeat an operator twice which improves the accuracy from first order to second order \cite{glowinski2017splitting}.
\end{remark}

\begin{remark}
	The Lie splitting scheme was used for three reasons: (1) The layer-by-layer structure of the Transformer naturally aligns with a sequential splitting approach, where each operator (attention, normalization, feedforward) is applied in a stepwise manner, mirroring the sequential flow of data through the encoder block. (2) Sequential splitting allows each substep to be mapped directly to a specific component of the Transformer. This enhances interpretability and facilitates the extension of the framework to variants such as Vision Transformers and Convolutional Transformers. (3) While parallel splitting could also be applied, they do not directly recover the standard Transformer’s serial layer composition. The sequential scheme is thus the most natural choice for reproducing the Transformer architecture. 
\end{remark}

\subsection{Solutions to Each Subproblem}
\label{sec.sol-to-sub}
In the following, we supply the details for the solution to each subproblem in the splitting scheme (\ref{eq.split.attention})--(\ref{eq.split.normal2}).

\subsubsection{Solution to Subproblem (\ref{eq.split.attention})}
We have
\begin{align}
	u^{1/M}=&u^{0} 
	+ \int_{\Omega_{x}}\soft_{2}\left( \frac{1}{\sqrt{|\Omega_y|}}\mylangle Q^0(\xb,\cdot;u^{0}),\ K^0(\eta,\cdot;u^{0};)\myrangle_{\Omega_y}\right) V^0(\eta,\yb;u^{0})d\eta.
	\label{eq.split.attention.tdis}
\end{align}
This step is a continuous version of the attention layer in \cite{vaswani2017attention}. It looks simple and uses two integral transformations to get the attention score through a softmax operator. This score function has values between $[0,1]$ and it is multiplied with the $V^0(\eta,\yb;u^{0})$, which are the extracted features from the input data through another integral transformation. 

\subsubsection{Solution to Subproblem (\ref{eq.split.normal1}) and (\ref{eq.split.normal2})}
\label{sec.normalization.dis}

For $u^{2/M}$ and $u^{1}$ in problem (\ref{eq.split.normal1}) and (\ref{eq.split.normal2}), they are in the form of
\begin{align}
	u-v=\partial I_{S_1(\sigma_1,\sigma_2)}(u)
	\label{eq.ln.gen}
\end{align}
for the given function $v(\xb,\yb)$ and constants $\sigma_1,\sigma_2$.
It is easy to derive that the solution $u$ of (\ref{eq.ln.gen}) solves the following problem: 
\begin{align}
	u=\argmin_{\bar{u}\in S_1(\sigma_1,\sigma_2)} \frac{1}{2} \|\bar{u}-v\|_{\Omega_y}^2.
	\label{eq.ln.min}
\end{align}
This shows that $u$ is a projection of $v$ to the set $S_1(\sigma_1,\sigma_2)$.  The following theorem gives a closed-form solution formula for $u$:
\begin{theorem}\label{thm.ln}
	The solution to problem (\ref{eq.ln.min}) is given as:
	\begin{align}
		&u(\xb,\yb)=\frac{v(\xb,\yb)-\alpha(\xb)}{\sqrt{\beta(\xb)}}\sigma_2+\sigma_1,
		\\
		\mbox{with } &\alpha(\xb;v)=\frac{1}{|\Omega_y|}\int_{\Omega_y} v(\xb,\xi)d\xi, \quad 
		\beta(\xb;v)=\frac{1}{|\Omega_y|}\int_{\Omega_y} (v(\xb,\xi)-\alpha(\xb))^2d\xi.
		\label{eq.ln.express}
	\end{align}
\end{theorem}
Theorem \ref{thm.ln} is proved in Appendix \ref{sec.thm.ln.proof}. We remark that the expression (\ref{eq.ln.express}) recovers the normalization layer in the continuous setting.  
According to Theorem \ref{thm.ln}, we see that the solutions  $u^{2/M}$ and $u^{(4+J)/M}$ are explicitly given as: 
\begin{align}
	&u^{2/M}(\xb,\yb)=\frac{u^{1/M}(\xb,\yb)-\alpha(\xb;u^{1/M})}{\sqrt{\beta(\xb;u^{1/M})}}\sigma_2^0+\sigma_1^0,\\
	&u^{1}(\xb,\yb)=\frac{u^{(3+J)/M}(\xb,\yb)-\alpha(\xb;u^{(3+J)/M})}{\sqrt{\beta(\xb;u^{(3+J)/M})}}\sigma_2^0+\sigma_1^0.
\end{align}

\subsubsection{Solution to Subproblem (\ref{eq.split.fnn}) and (\ref{eq.split.fnnfinal})}
\label{sec.fnn.dis}
Problem (\ref{eq.split.fnn}) is a semi-implicit equation for $u^{(2+j)/M}$. It can be solved exactly by the following sequential splitting
\begin{align}
	\begin{cases}
		\bar{u}^{(2+j)/M}=u^{(1+j)/M}+ \mylangle W_j^0(\cdot,\yb),\ u^{(1+j)/M}(\xb,\cdot) \myrangle_{\Omega_y} + b_j^0(\xb),\\
		u^{(2+j)/M}-\bar{u}^{(2+j)/M}=\partial I_{S_2}(u^{(2+j)/M}).
	\end{cases}
	\label{eq.split.fnn.seq}
\end{align}
The first equation in (\ref{eq.split.fnn.seq}) is linear in $\bar{u}^{(2+j)/M}$, corresponding to a linear layer in a feed-forward network in continuous settings. The second equation solves
\begin{align}
	u^{(2+j)/M}=\argmin_{v\in S_1} \int_{\Omega_x}\int_{\Omega_y} |v-\bar{u}^{(2+j)/M}|^2d\yb d\xb.
\end{align}
It is easy to see that its solution can be computed point-wisely:
\begin{align}
	u^{(2+j)/M}=\max\{\bar{u}^{(2+j)/M},0\}=\ReLU(\bar{u}^{(2+j)/M}),
\end{align}
which corresponds ReLU activation. As a result, Substep $2+j$ for $j=1,...,J-1$ is a feedforward network layer activated by ReLU. Similarly, Substep $2+J$ in (\ref{eq.split.fnnfinal}) is just a linear feedforward layer.

\subsection{Spatial Discretizations}
\label{sec.space.dis}
In this subsection, we discretize the solution discussed in Section \ref{sec.sol-to-sub} spatially. We will show that after spatial discretizations, the splitting scheme (\ref{eq.split.attention})--(\ref{eq.split.normal2}) exactly recovers the Transformer architecture in \cite{vaswani2017attention}. 

Suppose $\Omega_x=[0,L_x],\ \Omega_y=[0,L_y]$. We uniformly discretize $\Omega_x,\Omega_y$ by $N_x$ and $N_y$ grids in the $\xb$ and $\yb$ coordinates, respectively. For simplicity, we suppose $L_x=N_x, L_y=N_y$, leading to discretization steps $\dx=L_x/N_x=1, \ \dy=L_y/N_y=1$, and the grids by $\xb_k=k, \yb_k=k$.  
We use $\one_m$ to denote an $m$-dimensional vector with elements 1. For a matrix $\Wb\in \RR^{N_x\times N_y}$ and a vector $\bb\in \RR^{N_x}$, the matrix $\zb=(\Wb+\bb)\in \RR^{N_x\times N_y}$ is computed so that
$$
\zb_{k,l}=W_{k,l}+b_k.
$$
This is an easy notation to add a vector to a matrix in the discrete setting. 
With the spatial discretizations above, we define the discrete integrals (sum) as
\begin{align}
	\int_{\bar{\Omega}_x} f d\xb =  \sum_{k=1}^{N_x} f_{k}, \quad \int_{\bar{\Omega}_x} \int_{\bar{\Omega}_y}f (\xb,\yb) d\xb d\yb =  \sum_{k=1}^{N_x}\sum_{l=1}^{N_y} f_{k,l}.
\end{align}

For any functions $v(\xb),w(\xb)$ defined on $\Omega_x$, denote their discrete counterpart by $\vb,\wb$. We denote the discretized inner product between $\vb$ and $\wb$ as
\begin{align}
	\mylangle \vb,\wb\myrangle_{\bar{\Omega}_x}=\vb^{\top}\wb=\sum_{k=1}^{N_x} v_k w_k.
\end{align}
The discretized integrals and inner products $\langle \cdot,\cdot\rangle_{\bar{\Omega}_y}$ and $\langle \cdot,\cdot\rangle_{\bar{\Omega}_x \times \bar{\Omega}_y}$ are defined similarly.

\subsubsection{Computing the Discrete Solutions  $u^{1/M}$ in (\ref{eq.split.attention.tdis})}
Before we discretize the formula (\ref{eq.split.attention.tdis}), we first discretize some related quantities and operations.
To compute the attention scores, we discretize the integral transformations as follows:
\begin{align}
	\Qb^0(\ub^0) = \ub^0 \Wb^{Q,0}, \quad 
	\Kb^0(\ub^0) = \ub^0 \Wb^{K,0}, \quad 
	\Vb^0(\ub^0) = \ub^0 \Wb^{V,0}. \label{eq:dis-integral}
\end{align}
Here, $\ub^0$ is in $\mathbb{R}^{N_x \times N_y}$ and $\Wb^{Q,0}, \Vb^{Q,0}, \Kb^{Q,0}$ are in $\mathbb{R}^{N_y \times N_y}$, and the expressions such as $\ub^0 \Wb^{Q,0}$ represent standard matrix multiplications that yield new matrices in $\mathbb{R}^{N_x \times N_y}$. We emphasize that the multiplication used here is the conventional matrix product.
For $\Ab=\{a_{k,l}\}_{k,l}\in \RR^{N_x\times N_x}$, the discretized version of $\soft_2(\Ab)$ is given as
\begin{align}
	\left(\soft_{2,{\rm dis}}(\Ab)\right)_{k,l}=\frac{\exp(a_{k,l})}{ \sum_{l=1}^{N_x} \exp(a_{k,l})}. 
\end{align}
Then the updating formula (\ref{eq.split.attention.tdis}) for $u^{1/M}$ in the discrete setting is
\begin{align}
	\ub^{1/M}=\ub^0+ \soft_{2,{\rm dis}}\left(\Qb^0(\ub^0)(\Kb^0(\ub^0))^{\top}\right)\Vb^0(\ub^0).
	\label{eq.attention.dis}
\end{align}
We emphasis again that this is exactly the attention layer in \cite{vaswani2017attention}.  This is just a compact mathematical expression for it.  All the matrices $\Wb^{Q,0}, \Vb^{Q,0}, \Kb^{Q,0}$ will be learned in the training process. 

\subsubsection{Computing the Discrete Solutions  $u^{2/M}$ and $u^{1}$ in (\ref{eq.ln.gen})}
In Theorem \ref{thm.ln}, for any $\vb=\{v_{k,l}\}_{k,l}\in \RR^{N_x\times N_y}$, we discretize $\alpha(\xb)$ and $\beta(\xb)$ as
\begin{align}
	\alpha_k(\vb)=\frac{1}{L_y}\sum_{l=1}^{N_y} v_{k,l}, \ \beta_{k}(\vb)=\frac{1}{L_y}\sum_{l}(v_{k,l}-\alpha_k(\vb))^2
\end{align}
for $k=1,...,N_x$.
Then we compute the discrete solutions  $u^{2/M}$ and $u^{1}$ as
\begin{align}
	&u^{2/M}_{k,l}=\sigma_2^0\frac{u^{1/M}_{k,l}-\alpha_k(\ub^{1/M})}{\sqrt{\beta_k(\ub^{1/M})}} + \sigma_1^0, 
	\label{eq.ln1.dis}\\
	&u^{1}_{k,l}=\sigma_2^0\frac{u^{(3+J)/M}_{k,l}-\alpha_k(\ub^{(3+J)/M})}{\sqrt{\beta_k(\ub^{(3+J)/M})}} + \sigma_1^0, 
	\label{eq.ln2.dis}
\end{align}
for $k=1,...,N_x$ and $l=1,...,N_y$.

\subsubsection{Computing the Discrete Solutions  $u^{(3+J)/M}$ in (\ref{eq.split.skip2})}
For $u^{(3+J)/M}$, its discrete analogue is computed as
\begin{align}
	\ub^{(3+J)/M}= \frac{1}{2}(\ub^{(2+J)/M}+\ub^{2/M}).
	\label{eq.skip2.dis}
\end{align}
\subsubsection{Computing the Discrete Solutions  $u^{(2+j)/M}$ in (\ref{eq.split.fnn.seq})}
In (\ref{eq.split.fnn.seq}), we first compute
\begin{align}
	\bar{\ub}^{(2+j)/M}= \ub^{(1+j)/M} +  \ub^{(1+j)/M}\Wb_j^0+\bb^0_j,
	\label{eq.ffn.linear.dis}
\end{align}
and then update
\begin{align}
	\ub^{(2+j)/M}=\ReLU(\bar{\ub}^{(2+j)/M}).
	\label{eq.ffn.relu.dis}
\end{align}
The structure of the discretized scheme (\ref{eq.split.attention})--(\ref{eq.split.normal2}) is illustrated in Figure~\ref{fig.architecture}(b), and it exactly recovers the well-known Transformer architecture introduced in \cite{vaswani2017attention}, see the next section for detailed explanations.

\section{Mathematical Explanations of  Transformers}
\label{sec.connections}
\subsection{Connections to Transformer Encoder in \cite{vaswani2017attention}}
Transformer architecture was first proposed in \cite{vaswani2017attention} for natural language processing. 
After time and spatial discretization, our splitting scheme recovers the well-known Transformer with single-head attention in \cite{vaswani2017attention}. We supply the details of the explanations in the following. 
Let us  first show that one iteration $\ub^0\rightarrow \ub^{1}$ realizes one Transformer block:
\begin{itemize}
	\item Input $\ub^0$: With our discretization, $\ub^0\in \RR^{N_x\times N_y}$. It is an input with $N_x$ tokens. Each token is represented as a vector with $N_y$ numbers.
	
	\item Scaled dot-product attention: Equation (\ref{eq.attention.dis}) is a single head attention  with a skip connection. 
	The second term in the right-hand side of (\ref{eq.attention.dis}) is a single-head attention with query matrix $\Qb^0$, key matrix $\Kb^0$ and value matrix $\Vb^0$, whose weight matrices are $\Wb^{Q,0},\Wb^{K,0}$ and $\Wb^{V,0}$, respectively. It recovers the single-head attention in \cite{vaswani2017attention}. The first term in the right-hand side is a skip connection, adding together the input and output of attention.
	
	\item Layer normalization: Equation (\ref{eq.ln1.dis})--(\ref{eq.ln2.dis}) perform layer normalization. For each token, (\ref{eq.ln1.dis})--(\ref{eq.ln2.dis}) normalize the input along the embedding direction to have mean $\sigma_1$ and variance $\sigma_2^2$.
	
	\item Position-wise feedforward network: Equation (\ref{eq.ffn.linear.dis})--(\ref{eq.ffn.relu.dis}) realizes a feedforward network with $J$ layers activated by ReLU. In particular, (\ref{eq.ffn.linear.dis}) can be rewritten as
	\begin{align*}
		\bar{\ub}^{(2+j)/M}=\ub^{(1+j)/M}(\Ib+\Wb_j^0)+\bb^0_j,
	\end{align*}
	where $\Ib$ denotes the identity matrix.
	It is a linear layer with weight matrix $(\Ib+\Wb_j^0)$ and bias $\bb^0_j$.
	Set $J=2$. Equation (\ref{eq.ffn.linear.dis})--(\ref{eq.ffn.relu.dis}) recovers the feedforward network in \cite{vaswani2017attention}.
	
	\item Skip connection: Equation (\ref{eq.skip2.dis}) is a relaxation step in the splitting scheme. It realizes the skip connection by averaging the input and output of the feedforward network.
\end{itemize}

The discussion above shows that one time step of the splitting scheme supplied in Section \ref{sec.timedis}  realizes the Transformer of \cite{vaswani2017attention}. In Section \ref{sec.timedis}, the time domain is discretized into $N_t$ time steps. Thus, the whole operator-splitting scheme is equivalent to compositions of $N_t$ Transformer blocks. Since the learnable variables $\theta^n=\{\Wb^{Q,n},\Wb^{K,n},\Wb^{V,n},\{\Wb_j^n,\bb_j^n\}_{j=1}^J\}$ depends on time, they are different for different time steps. Setting $N_t=6$ recovers the architecture in \cite{vaswani2017attention}. Furthermore, solving (\ref{eq.loss.dis}) for the control variables is equivalent to training the network.

\subsection{Connections to ViT in \cite{dosovitskiyimage}}
\label{sec.vit}
Based on the Transformer encoder in \cite{vaswani2017attention}, vision Transformer (ViT) was proposed in \cite{dosovitskiyimage} for image classification. In ViT, each image is cropped into patches, each of which is taken as a token. ViT has a similar architecture as the Transformer encoder in \cite{vaswani2017attention}, except that it has an additional input embedding layer with learnable parameters, and a linear layer at the end for output. 

In order to use (\ref{eq.split.attention})--(\ref{eq.split.normal2}) to provide a mathematical explanation of ViT, we need to incorporate the additional embedding layer and last linear layer with our algorithm. It can be easily done by introducing data-driven data pre-processing and post-processing steps. For the additional input embedding, we design a pre-processing step to generate the initial condition $f$. Consider the discrete setting. Denote $\Rb\in \RR^{(N_x-1)\times D}$ as the collection of $(N_x-1)$ patches of an image, where each row of $\Rb$ corresponds to one flattened patch of dimension $D$. Let $\Eb\in \RR^{D\times N_y}$ be an embedding matrix and $\vb_0\in \RR^{1\times N_y}$ be a vector where both are learnable. The vector $\vb_0$ represents the extra learnable embedding. We generate $\ub_0$ as
\begin{align}
	\ub_0=F(\Rb)=\begin{bmatrix}
		\vb_0\\ \Rb \Eb
	\end{bmatrix}.
\end{align}
The last linear layer in ViT can be taken as a data post-processing. For any $\ub\in \RR^{N_x\times N_y}$ and weight matrix $\Bb\in \RR^{N_y\times d}$, define
\begin{align}
	G(\Bb;\ub)=\begin{bmatrix}
		1 & \mathbf{0}_{1\times (N_x-1)}
	\end{bmatrix}\ub \Bb,
\end{align}
where $\mathbf{0}_{1\times (N_x-1)}$ is a zero vector of size $(N_x-1)$. Let $\bar{\theta}, \bar{\cN}$ be defined as in Section \ref{sec.learning.timedis}. We denote
$$
\widetilde{\theta}=\{ \bar{\theta},\vb_0,\Eb,\Bb\}
$$
as the collection of all learnable parameters. With the data pre- and post-processing, we define the whole operation 
\begin{align}
	\bar{\cN}^{\rm ViT}_{\widetilde{\theta}}=G\circ \bar{\cN}\circ F.
\end{align}
One can readily check that $\bar{\cN}^{\rm ViT}_{\widetilde{\theta}}$ recovers the ViT in \cite{dosovitskiyimage}.
Given a dataset $\{(\ub_i,\vb_i)\}_{i=1}^B$, we learn all parameters by solving
\begin{align}
	\min_{\widetilde{\theta}} \frac{1}{B}\sum_{i=1}^B\ell(\bar{\cN}^{\rm ViT}_{\widetilde{\theta}}(\ub_i),\vb_i),
	\label{eq.loss.dis.vit}
\end{align}
which is equivalent to training a ViT.

\section{Multi-Head Attention}
\label{sec.multihead}
As discussed in Section \ref{sec.connections}, the proposed splitting scheme for problem (\ref{eq.control}) is the Transformer architecture with one-head attention. In this section, we extend (\ref{eq.control}) and splitting scheme (\ref{eq.split.attention})--(\ref{eq.split.normal2}) so that it realizes multi-head attention. 

\subsection{Control Problem with Continuous Head Dimension}
We first consider a continuous-head setting. After discretizations, we will show that it recovers multi-head attention. In the continuous-head setting, we take head as an additional dimension. To accommodate the new dimension, we update the definition of $W^{Q},W^{K},W^{V}$ to $W^{Q}(\yb,\yt,h,t),W^{K}(\yb, \yt,h,t),W^{V}(\yb,\yt,h,t)$ on $\Omega_y\times\Omega_y\times \Omega_h\times [0,+\infty)$, where $h$ is the variable for the head dimension with domain $\Omega_h$. Similar to (\ref{eq.atten.trans1})--(\ref{eq.atten.trans3}), we define
\begin{align}
	&Q(\xb,\yb,h,t;u)=\mylangle W^{Q}(\cdot,\yb,h,t),\ u(\xb,\cdot,t)\myrangle_{\Omega_y},\nonumber\\
	&K(\xb,\yb,h,t;u)=\mylangle W^{K}(\cdot,\yb,h,t),\ u(\xb,\cdot,t)\myrangle_{\Omega_y}, \nonumber\\
	&V(\xb,\yb,h,t;u )=\mylangle W^{V}(\cdot,\yb,h,t),\ u(\xb,\cdot,t)\myrangle_{\Omega_y}.
\end{align}
We update $\gamma(\xb,\widetilde{\xb},t;u )$ in (\ref{eq.scale.soft}) to
\begin{align}
	\gamma(\xb,\widetilde{\xb},h,t;u )=\text{Softmax}_{2}\left(\frac{1}{\sqrt{|\Omega_y|}}\mylangle Q(\xb,\cdot,h,t;u  ),K(\widetilde{\xb},\cdot,h,t;u )\myrangle_{\Omega_y}\right).
	\label{eq.scale.soft.mh}
\end{align}
The control problem (\ref{eq.control}) is updated to
\begin{align}
	\begin{cases}
		u_t=\underbrace{\int_{\Omega_h}\mylangle\gamma(\xb,\cdot,h,t;u),\ V(\cdot,\yb,h,t; u )\myrangle_{\Omega_x} dh}_{\mbox{I-a: attention} } + \underbrace{\partial I_{S_{1}(\sigma_1,\sigma_2)}(u)}_{\mbox{II: layer normalization}} \\
		\hspace{1cm}+ \underbrace{\sum_{j=1}^{J}\left(\mylangle W_j(\cdot,\yb,t),u(\xb,\cdot,t) \myrangle_{\Omega_y}+ b_j(\xb,t)\right)
			+ \partial I_{S_2}(u)}_{\mbox{III: fully connected network}} \mbox{ for } t\in(0,T],\\
		u(\xb,\yb,0)=f(\xb,\yb).
	\end{cases}
	\label{eq.control.mh}
\end{align}
Compared to (\ref{eq.control}), problem (\ref{eq.control.mh}) replaces I by I-a, which introduces an additional integral with respect to $h$.

\subsection{Time Discretizations Through Operator-Splitting Method}
The time discretizations and operator-splitting scheme for problem (\ref{eq.control.mh}) are the same as those for problem (\ref{eq.control}) except that we update Substep 1 to
\begin{align}
	&u^{1/M}-u^{0}= \int_{\Omega_h}\mylangle \gamma^0(\xb,\cdot,h;u^{0}),V^0(\cdot,\yb,h;u^{0})\myrangle_{\Omega_x} dh,
	\label{eq.split.attention.mh}
\end{align}
whose solution is given as
\begin{align}
	&u^{1/M}= \nonumber\\
	&u^{0}+ \int_{\Omega_h}\int_{\Omega_{x}}\soft_{2}\left( \frac{1}{\sqrt{|\Omega_y|}}\mylangle Q^0(\xb,\cdot,h;u^{0}),K^0(\eta,\cdot,h;u^{0};)\myrangle_{\Omega_y}\right) V^0(\eta,\yb,h;u^{0})d\eta dh.
	\label{eq.split.attention.mh.tdis}
\end{align}
\subsection{ Discrete Solutions  of $u^{1/M}$ in (\ref{eq.split.attention.mh.tdis})}
Suppose $\Omega_h=[0,L_h]$. In addition to the spatial discretization discussed in Section \ref{sec.space.dis}, we discretize the head domain into $N_h$ grids. For simplicity, we suppose $L_h=N_h$.

For function $W^{K,0}(\yb,\yt,h)$, we denote $W^{Q,0}_{k,l,m}=W^{Q,0}(k,l,m), \Wb^{Q,0}_{m}=\{W^{Q,0}_{k,l,m}\}_{k,l}\in \RR^{N_y\times N_y }$ and $\Wb^{Q,0}=\{\Wb^{Q,0}_{m}\}_{m}\in \RR^{N_y\times N_y \times N_h}$. We define $\Wb^{K,0},\Wb^{V,0}$ similarly. We denote the discretized integral transformation in the updated attention score as
\begin{align*}
	&\Qb^0_{m}(\ub^0)=\ub^0 \Wb^{Q,0}_m,\ \Kb^0_{m}(\ub^0)=\ub^0 \Wb^{K,0}_m, \ \Vb^n_{m}(\ub^0)=\ub^0 \Wb^{V,0}_m,
\end{align*}
and denote $\Qb^0(\ub)=\{\Qb^0_{m}(\ub^0)\}_{m}, \ \Kb^0(\ub^0)=\{\Kb^0_{m}(\ub^0)\}_{m}, \ \Vb^0(\ub^0)=\{\Vb^0_{m}(\ub^0)\}_{m}\in \RR^{N_x\times N_y \times N_h}$.
The inner product $\mylangle Q^0(\xb,\cdot,h;u^0  ),K^0(\widetilde{\xb},\cdot,h;u^0 )\myrangle_{\Omega_y}$ is discretized as
\begin{align}
	\mylangle \Qb^0_m(\ub^0),\Kb^0_m(\ub^0)\myrangle_{\bar{\Omega}_y}=\Qb^0_m(\ub^0)(\Kb^0_m(\ub^0))^{\top}.
\end{align}
Then the updating formula (\ref{eq.split.attention.mh.tdis}) for $u^{1/M}$ is discretized as
\begin{align}
	\ub^{1/M}=
	\ub^0+ \sum_{m=1}^{N_h}\soft_{2,{\rm dis}}\left(\frac{1}{\sqrt{L_x}} \Qb^0_m(\ub^0)(\Kb^0_m(\ub^0))^{\top}\right) \Vb^0_m(\ub^0).
	\label{eq.attention.mh.dis}
\end{align}

\subsection{Connections to Transformers}
Similar to the discussion in Section \ref{sec.connections}, with (\ref{eq.attention.mh.dis}), the operator-splitting scheme for the control problem (\ref{eq.control.mh}) realizes $N_t$ times encoder in \cite{vaswani2017attention} with multi-head attention. To show this, we only need to demonstrate the equivalence between (\ref{eq.attention.mh.dis}) and the multi-head attention.

The second term in the right-hand side of (\ref{eq.attention.mh.dis}) sums over the embedding dimension and head dimension. It is a multi-head attention with $N_h$ heads, query matrix $\Qb^0$, key matrix $\Kb^0$ and value matrix $\Vb^0$, whose weight matrices are $\Wb^{Q,0},\Wb^{K,0}$ and $\Wb^{V,0}$, respectively.

\section{Convolution Transformers for Video and Image Problems}
\label{sec.convolution}

Image and video data inherently possess rich local spatial structures, which have been effectively exploited by convolutional neural networks (CNNs) \cite{ronneberger2015u}. 
Convolutions serve as a powerful mechanism for capturing spatially localized patterns, making them especially well-suited for data defined on Cartesian grids.

In contrast, Transformer architectures were originally developed for language processing, where the data—represented by functions such as \( u(\xb, \yb, t) \)—often lack explicit spatial structure, as in the case of sentences or token sequences. For such unstructured data, attention-based mechanisms leveraging general integral transforms offer a flexible and effective modeling tool.

However, when applying Transformers to structured data such as images, videos, or other signals on regular grids, it has been shown that convolutional operations are significantly more efficient for feature extraction. This is due to their locality, weight sharing, and computational efficiency. Notably, convolutions can be interpreted as a special class of integral transformations with translation-invariant kernels and localized support.

Thanks to the integral formulation of the \( Q \), \( K \), and \( V \) operators in our framework (Sections~\ref{sec.model}--\ref{sec.dis}), it becomes straightforward to specialize these operators to convolutions. This allows our model to seamlessly incorporate convolutional structures, thereby leveraging domain-specific inductive biases and enhancing computational performance when applied to grid-structured data.

\subsection{Convolution Transformers}
\label{sec.convTrans}
We discuss the case for grayscale images. The framework can be extended to color images and videos by introducing additional variables to $u$.

We consider functions of interest $u(x,\yb,t)$ where $x\in \Omega_x\subset \RR $ is the index of tokens (image patches), $\yb \in \Omega_y\subset \RR^2$ are variables for spatial domain (indices of pixel location), and $t$ is the time. For each token and any fixed time $t$, $u$ is a two-dimensional image. 

In this setting, we define $W^Q(\yb,t),W^K(\yb ,t),W^V(\yb,t)\in \Omega_y\times [0,+\infty) $ as convolution kernels. Operations $Q,K,V$ are defined by covolution as
\begin{align}
	&Q(x,\yb,t;u)=W^{Q}(\cdot,t)*u(x,\cdot,t)=\int_{\Omega_y}W^{Q}(\xxi,t)u(x, \yb -\xxi,t)d\xxi ,\nonumber\\
	&K(x,\yb,t;u)=W^{K}(\cdot,t)*u(x,\cdot,t)=\int_{\Omega_y}W^{K}(\xxi,t)u(x,\yb-\xxi, t)d\xxi,\nonumber\\
	&V(x,\yb,t;u)=W^{V}(\cdot,t)*u(x,\cdot,t)=\int_{\Omega_y}W^{V}(\xxi ,t)u(xb,\yb-\xxi,t)d\xxi .
	\label{eq.conv.atten.trans}
\end{align}

The score function is computed as
\begin{align}
	\gamma(x,\widetilde{x},t ;u)=\text{Softmax}_{2}\left(\frac{1}{\sqrt{|\Omega_y|}}\int_{\Omega_y} Q(x,\xxi,t;u  )K(\widetilde{x},\xxi,t;u )d\xxi\right),
	\label{eq.conv.scale.soft}
\end{align}
and the output of the attention layer is then given by: $\mylangle \gamma(x,\cdot,t;u),V(\cdot,\yb,t;u)\myrangle_{\Omega_x}.$

With the modification above, scheme (\ref{eq.split.attention})--(\ref{eq.split.fnn}) gives rise to convolutional Transformer. Solvers discussed in Section \ref{sec.sol-to-sub} can still be used to solve subproblems in the new scheme.

\subsection{Connections to CvT in \cite{wu2021cvt}}
Convolutional vision Transformer (CvT) \cite{wu2021cvt} incorporates Vit with convolutional neural networks (CNN), which utilizes both image local spatial structures (by CNN) and global contexture (by Transformer). CvT consists of two components, convolutional token embedding and convolutional Transformer block. The convolution Transformer discussed in Section  \ref{sec.convTrans} realizes a one-head CvT with one stage and without convolutional token embedding.

In fact, the convolutional token embedding is an additional convolution layer. Our operator-splitting scheme can be easily revised to accomodate this layer.
Let $W^C(\xb,\yb,t)$ be a convolution kernel and $b^C(\xb,t)$ be a bias term.  We consider the following control problem
\begin{align}
	\begin{cases}
		u_t=\underbrace{\mylangle \gamma(\xb,\cdot,t;u),\ V(\cdot,\yb,t;u)\myrangle_{\Omega_x} }_{\mbox{I: attention} } + \underbrace{\partial I_{S_{1}(\sigma_1(t),\sigma_2(t))}(u)}_{\mbox{II: layer norml.}} + \underbrace{W^C(\cdot,t)*u(\xb,\cdot,t)+b^C(\xb,t)}_{\mbox{IV: conv. token embedding}}\\
		\hspace{1cm}+ \underbrace{\sum_{j=1}^{J}\left(\mylangle W_j(\cdot,\yb,t), \ u(\xb,\cdot,t)\myrangle_{\Omega_y}+ b_j(\xb,t)\right)
			+ \partial I_{S_2}(u)}_{\mbox{III: fully connected network}}, \quad t\in (0,T], \\
		u(\xb,\yb,0)=f(\xb,\yb),
	\end{cases}
	\label{eq.control.cvt}
\end{align}
where the score function $\gamma$ is computed as discussed in Section \ref{sec.convTrans}. For the operator-splitting scheme (\ref{eq.split.attention})--(\ref{eq.split.fnn}), we add two additional substeps before (\ref{eq.split.attention}). Given $u^0$, we compute $u^1$ by $M=6+J$ substeps:

\noindent {\bf Substep 1}:
Solve $\widetilde{u}^{1/M}$ from
\begin{align}
	&u^{1/M}-u^{0}= W^{C,0}(\cdot)*u^0(\xb,\cdot)+b^{C,0}(\xb).
	\label{eq.split.cte}
\end{align}

\noindent {\bf Substep 2}:  Compute $u^{2/M}$ from
\begin{align}
	u^{2/M}-u^{1/M}=\partial I_{S_1(\sigma_1^0,\sigma_2^0)}(u^{2/M}).
	\label{eq.split.normal.cte}
\end{align}
The remaining intermediate variables $u^{3/M},...,u^1$ can be computed via (\ref{eq.split.attention})--(\ref{eq.split.fnn}), respectively. In spatial discretization, $N_x$ corresponds to the number of tokens, i.e., the number of convolution channels in CvT. Equipped with data pre- and post-processing as discussed in Section \ref{sec.vit}, the discretized new operator-splitting scheme recovers a one-head one-stage CvT. 

For a $K$-stage CvT, the number of tokens changes from stage to stage. This property can be realized by incorporating the new operator-splitting scheme with a hybrid splitting scheme \cite{tai2024pottsmgnet}. Specifically, we need to define $K$ convolution kernels and a bias for convolutional token embedding. Based on the multigrid idea, these kernels are discretized at different scales, corresponding to various numbers of tokens.

\section{Conclusion}
\label{sec.conclusion}
In this paper, we have introduced a novel operator-theoretic framework that interprets the Transformer architecture as a discretized operator-splitting scheme for a continuous integro-differential equation. Within this formulation, core components of the Transformer—such as self-attention, layer normalization, and feedforward networks—are rigorously derived as substeps in a structured variational process and operator-splitting techniques. This approach establishes a unified mathematical foundation that applies broadly to Transformer-based models, including the original Transformer \cite{vaswani2017attention}, Vision Transformer (ViT) \cite{dosovitskiyimage}, and Convolutional vision Transformer (CvT) \cite{wu2021cvt}. By bridging deep learning architectures with continuous mathematical modeling, our framework advances theoretical understanding and offers a principled pathway for the design, analysis, and control of future neural architectures. 

Our continuous formulation offers a principled basis for deriving new architectures through alternative numerical discretization schemes, potentially leading to models with improved stability. It also enables the use of PDE theory to analyze Transformer properties such as stability, approximation power, and convergence. Furthermore, framing the network as an optimal control problem links deep learning to control theory, which may inspire new optimization algorithms and insights into training dynamics.

Our framework currently considers ReLU activation in the feedforward layer and only focuses on the transformer block of \cite{vaswani2017attention}. Following this work, possible future  directions include generalizing this framework to incorporate general activation functions, incorporation of discrete inductive biases such as positional encodings into the continuous formulation, and a more rigorous analysis of the well-posedness and regularity of the underlying integro-differential equation.

We believe the continuous operator perspective introduced here not only deepens the theoretical understanding of Transformers but also provides a flexible and principled foundation for the next generation of attention-based models. By continuing to explore the intersection of continuous dynamical systems, operator splitting, and deep learning, we hope to further bridge the gap between applied architecture design and rigorous mathematical analysis.

\appendix
\section*{Appendix}
\section{Explanation of the variables for Transformers}
\label{appendix:A}
The input to the Transformer model is usually a sequence. In the context of language processing, the input is usually sentences. Transformers do not process raw text directly. Instead, they break sentences into tokens. A token can be a word (e.g., "hello"), a subword (e.g., "multi" + "grid" = "multigrid"), or even a single character.
Then, each different token is transformed into unique embedding vectors through an embedding layer.
The way text is tokenized affects how efficiently a model understands language.

For example, we have an input sentence "Life is like a boat", and we treat each word as a token. The embedding layer would transform each token into a $512$-dimensional vector. Then, the input sentence is represented by a matrix $U$ of size $5\times 512$:
$$\underset{\text{Input sentence}}{\underbrace{\text{Life is like a boat}}}\rightarrow\underset{\text{Tokens}}{\underbrace{[\text{"Life","is","like","a","boat"}]}}\rightarrow
\underset{\text{Embedded matrix}}{\underbrace{U=\begin{bmatrix}
			\text{Embedding("Life")} \\
			\text{Embedding("is")} \\
			\text{Embedding("like")} \\
			\text{Embedding("a")} \\
			\text{Embedding("boat")}
\end{bmatrix}}}.
$$
In general, a sentence of $N_x$ tokens with $N_y$-dimensional embedding would be transformed into a matrix $U$ of size $N_x\times N_y$. This matrix $U$ can be viewed as the discretization of a continuous function $u(\mathbf x,\mathbf y, t):\Omega_x\times\Omega_y\rightarrow\mathbb R\times [0,T)$, where $\Omega_x$ and $\Omega_y$ are continuous domains, on a structured grid of size $N_x\times N_y$. Here, the variable $\mathbf x$ indicates the position of tokens,  $\mathbf y$ indicates the entry of token vectors, and $t$ indicates the time variable. The embedded matrix $U$ is fed into the Transformer model and updated by attention layers.

\section{Proof of Theorem \ref{thm.ln}}
\label{sec.thm.ln.proof}
\begin{proof}[Proof of Theorem \ref{thm.ln}]
	Problem (\ref{eq.ln.min}) can be rewritten as
	\begin{align}
		\begin{cases}
			u(\xb,\yb)=\argmin_{\bar{u}} \frac{1}{2} \|\bar{u}-v\|_{\Omega_y}^2,\\
			\frac{1}{|\Omega_y|}\int_{\Omega_y} u(\xb,\xi)d\xi=\sigma_1,\\
			\frac{1}{|\Omega_y|}\int_{\Omega_y} (u(\xb,\xi)-\sigma_1)^2d\xi=\sigma_2^2.
		\end{cases}
	\end{align}
	Introducing Lagrange multipliers $\lambda_1,\lambda_2$, the minimizers for the above problem satisfy:
	\begin{align}
		\begin{cases}
			(u-v)+\frac{\lambda_1}{|\Omega_y|}+ \frac{2\lambda_2}{|\Omega_y|}(u-\sigma_1)=0,\\
			\frac{1}{|\Omega_y|}\int_{\Omega_y} u(\xb,\yb)d\yb=\sigma_1,\\
			\frac{1}{|\Omega_y|}\int_{\Omega_y} (u(\xb,\yb)-\sigma_1)^2d\yb=\sigma_2^2.
		\end{cases}
		\label{eq.ln.proof.lagrange}
	\end{align}
	Integrate the first equation over $\yb$ gives rise to
	\begin{align}
		0= \int_{\Omega_y} u(\xb,\xi)d\xi -\int_{\Omega_y} v(\xb,\xi)d\xi + \lambda_1 
		= |\Omega_y|\sigma_1 -\int_{\Omega_y} v(\xb,\xi)d\xi+\lambda_1,
	\end{align}
	implying that 
	\begin{align}
		\lambda_1=\int_{\Omega_y} v(\xb,\xi)d\xi- |\Omega_y|\sigma_1 .
		\label{eq.ln.proof.lam1}
	\end{align}
	Multiplying the first equation in (\ref{eq.ln.proof.lagrange}) by $(u-\sigma_1)$ and then integrating over $\yb$ give rise to
	\begin{align}
		\int_{\Omega_y} (u(\xb,\xi)-v(\xb,\xi))(u(\xb,\xi)-\sigma_1)d\xi +2\lambda_2\sigma_2^2=0,
	\end{align}
	from which we deduce
	\begin{align}
		\int_{\Omega_y} (u^2(\xb,\xi)-u(\xb,\xi)\sigma_1)d\xi- \int_{\Omega_y} v(\xb,\xi)(u(\xb,\xi)-\sigma_1) d\xi+ 2\lambda_2\sigma_2^2=0.
		\label{eq.ln.proof.lam2.1}
	\end{align}
	Note that
	\begin{multline}
		\int_{\Omega_y} (u^2(\xb,\xi)-u(\xb,\xi)\sigma_1)d\xi= \int_{\Omega_y} (u^2(\xb,\xi)-2u(\xb,\xi)\sigma_1 + \sigma_1^2)d\xi 
		\\
		= \int_{\Omega_y} (u(\xb,\xi)-\sigma_1)^2d\xi 
		= |\Omega_y|\sigma_2^2.
		\label{eq.ln.proof.meanvar}
	\end{multline}
	Substituting (\ref{eq.ln.proof.meanvar}) into (\ref{eq.ln.proof.lam2.1}) gives 
	\begin{align}
		\lambda_2=-\frac{|\Omega_y|}{2}+\frac{\int_{\Omega_y} v(\xb,\xi)(u(\xb,\xi)-\sigma_1) d\xi}{2\sigma_2^2}.
		\label{eq.ln.proof.lam2}
	\end{align}
	Substituting (\ref{eq.ln.proof.lam1}) and (\ref{eq.ln.proof.lam2}) into the first equation of (\ref{eq.ln.proof.lagrange}), we get
	\begin{align}
		0=&(u-v)+\frac{\lambda_1}{|\Omega_y|}+ \frac{2\lambda_2}{|\Omega_y|}(u-\sigma_1) \nonumber\\
		=&(u-v)+\frac{1}{|\Omega_y|}\int_{\Omega_y} v(\xb,\xi)d\xi- \sigma_1+ \left(-1+\frac{\int_{\Omega_y} v(\xb,\xi)(u(\xb,\xi)-\sigma_1) d\xi}{|\Omega_y|\sigma_2^2}\right)u \nonumber\\
		&- \left(-1+\frac{\int_{\Omega_y} v(\xb,\xi)(u(\xb,\xi)-\sigma_1) d\xi}{|\Omega_y|\sigma_2^2}\right)\sigma_1 \nonumber\\
		=&\frac{\int_{\Omega_y} v(\xb,\xi)(u(\xb,\xi)-\sigma_1) d\xi}{|\Omega_y|\sigma_2^2}(u-\sigma_1)-v+\frac{1}{|\Omega_y|}\int_{\Omega_y} v(\xb,\xi)d\xi.
	\end{align}
	We have
	\begin{align}
		\frac{\int_{\Omega_y} v(\xb,\xi)(u(\xb,\xi)-\sigma_1) d\xi}{|\Omega_y|\sigma_2^2}(u-\sigma_1)=v-\frac{1}{|\Omega_y|}\int_{\Omega_y} v(\xb,\xi)d\xi.
		\label{eq.ln.proof.ueq}
	\end{align}
	Define $\alpha(\xb),\beta(\xb)$ as in Theorem \ref{thm.ln}. We claim that the solution to  (\ref{eq.ln.proof.ueq}) is given as
	\begin{align}
		u=\frac{v-\alpha}{\sqrt{\beta}}\sigma_2+\sigma_1.
		\label{eq.ln.proof.uexpr}
	\end{align}
	To verify it, note that (\ref{eq.ln.proof.uexpr}) can be rewritten as
	\begin{align}
		\frac{u-\sigma_1}{\sigma_2}=\frac{v-\alpha}{\sqrt{\beta}}.
		\label{eq.ln.proof.uexpr2}
	\end{align}
	Utilizing (\ref{eq.ln.proof.uexpr2}), we deduce
	\begin{align}
		&\frac{\int_{\Omega_y} v(\xb,\xi)(u(\xb,\xi)-\sigma_1) d\xi}{|\Omega_y|\sigma_2^2}(u-\sigma_1) \nonumber\\
		=& \frac{1}{|\Omega_y|}\left(\int_{\Omega_y} v(\xb,\xi)\frac{(u(\xb,\xi)-\sigma_1)}{\sigma_2} d\xi\right) \frac{u-\sigma_1}{\sigma_2}  \nonumber\\
		=&\frac{1}{|\Omega_y|}\left(\int_{\Omega_y} v(\xb,\xi)\frac{(v(\xb,\xi)-\sigma_1)}{\sqrt{\beta}} d\xi\right) \frac{v-\sigma_1}{\sqrt{\beta}}  \nonumber\\
		=&\frac{1}{|\Omega_y|\beta}\left(\int_{\Omega_y} [v^2(\xb,\xi)-2\sigma_1v(\xb,\xi)+\sigma_1^2]d\xi\right) (v-\sigma_1) \nonumber\\
		=&\frac{1}{|\Omega_y|\beta}\left(\int_{\Omega_y} (v(\xb,\xi)-\sigma_1)^2d\xi\right) (v-\sigma_1) \nonumber\\
		=&\frac{1}{|\Omega_y|\beta}(|\Omega_y|\beta)(v-\sigma_1) \nonumber\\
		=&v-\frac{1}{|\Omega_y|}\int_{\Omega_y} v(\xb,\xi) d\xi,
	\end{align}
	where the second equality follows from (\ref{eq.ln.proof.uexpr2}). The derivation above verifies that (\ref{eq.ln.proof.uexpr}) is the solution to (\ref{eq.ln.proof.ueq}).
\end{proof}

\section{An Introduction to Operator Splitting Schemes}
\label{appendix:c} 

Operator-splitting methods decompose a complicated time-evolution problem into several substeps so that each substep can be solved efficiently. Consider the evolution equation
\begin{align}
	\begin{cases}
		u_t+\sum_{k=1}^K A_k(t;u)=0 \mbox{ in } \Omega\times(0,T],\\
		u(0)=u_0,
	\end{cases}
\end{align}
where $A_k(t;u)$'s are operators applied on $u$. Denote $t^n=n\Delta t$.
Given $u^n$, there are three types of splitting strategies to compute $u^{n+1}$.

Sequential Lie scheme  \cite[Sec. 2.2]{glowinski2016some}  decomposes this computation into $K$ sequential substeps. Given $u^n$, one computes $u^{n+1}$ as follows:\\
For $k=1,...,K$, solve
\begin{align}
	\begin{cases}
		u_t+A_k(t;u)=0 \mbox{ in } \Omega\times (t^n,t^{n+1}],\\
		u(t^n)=u^{n+(k-1)/K},
	\end{cases}
	\label{eq.splitting.seq}
\end{align}
and set $u^{n+k/K}=u(t^{n+1})$.

Parallel splitting \cite{lu1992parallel} decomposes the computation into $K$ parallel substeps:\\
For $k=1,...,K$, compute
\begin{align}
	\begin{cases}
		v_t+KA_k(t;v)=0 \mbox{ in } \Omega\times (t^n,t^{n+1}],\\
		v(t^n)=u^n,
	\end{cases}
	\label{eq.splitting.paral}
\end{align}
and set $v_k=u(t^{n+1})$. Then compute
$$
u^{n+1}=\frac{1}{K}\sum_{k=1}^K v_k.
$$

Hybrid splitting \cite{tai2024pottsmgnet} is a combination of sequential and parallel splitting, which decomposes the computation into sequential substeps, each of which contains several parallel substeps. 

Substep (\ref{eq.splitting.seq}) and (\ref{eq.splitting.paral}) can be solved by one-step explicit or implicit scheme. For example, (\ref{eq.splitting.seq}) can be solved by 
\begin{align*}
	\frac{u^{n+k/K}-u^{n+(k-1)/K}}{\Delta t}+A_k(u^{n+k/K})=0 \ 
	\mbox{   or   } \
	\frac{u^{n+k/K}-u^{n+(k-1)/K}}{\Delta t}+A_k(u^{n+(k-1)/K})=0.
\end{align*}
Our scheme in Section \ref{sec.timedis} is a sequential splitting scheme, in which (\ref{eq.split.attention}) uses explicit scheme, and (\ref{eq.split.normal1}) and (\ref{eq.split.normal2}) use implicit scheme. Substep (\ref{eq.split.fnn}) is semi-implicit, which is slightly different from others. In fact, (\ref{eq.split.fnn}) is a composition of two substeps, as it is solved by another sequential splitting (\ref{eq.split.fnn.seq}). In (\ref{eq.split.fnn.seq}), the first substep is explicit and the second one is implicit.

It can be shown that with proper conditions, all three strategies are first-order accurate schemes for the evolution equation.

\bibliographystyle{abbrv}
\bibliography{ref}
\end{document}